\documentclass{article}

\usepackage[preprint,nonatbib]{neurips_2024} 

\usepackage[utf8]{inputenc} 
\usepackage[T1]{fontenc}    
\usepackage{hyperref}       
\usepackage{url}            
\usepackage{booktabs}       
\usepackage{amsfonts}       
\usepackage{nicefrac}       
\usepackage{microtype}      
\usepackage{xcolor}         

\usepackage{amsmath,amssymb,amsthm,mathtools,bbm}

\usepackage[detect-all]{siunitx}

\usepackage{cleveref}

\usepackage{enumitem}

\newtheorem{theorem}{Theorem}[section]
\newtheorem{lemma}[theorem]{Lemma}
\newtheorem{corollary}[theorem]{Corollary}
\newtheorem*{example}{Example}

\newtheorem{innercustomgeneric}{\customgenericname}
\providecommand{\customgenericname}{}
\newcommand{\newcustomtheorem}[2]{%
  \newenvironment{#1}[1]
  {%
   \ifdefined\crefalias\crefalias{innercustomgeneric}{#2}\fi
   \renewcommand\customgenericname{#2}%
   \renewcommand\theinnercustomgeneric{##1}%
   \innercustomgeneric
  }
  {\endinnercustomgeneric}%
  \ifdefined\crefname\crefname{#2}{#2}{#2s}\fi
}
\newcustomtheorem{customlemma}{Lemma}
\newcustomtheorem{customcorollary}{Corollary}

\usepackage{algorithm}
\usepackage{algpseudocode}

\newcommand{\R}{\mathbb{R}}
\newcommand{\N}{\mathbb{N}}

\newcommand{\Eq}[1]{\mathbb{E}_q\!\left[#1\right]}
\newcommand{\KL}{D_{\text{KL}}}
\newcommand{\KLlambda}{D_{\text{KL}}^\lambda}

\newcommand{\rom}[1]{\uppercase\expandafter{\romannumeral #1\relax}}
\DeclareMathOperator{\Poisson}{Poisson}
\DeclareMathOperator{\Pois}{Pois}
\DeclareMathOperator{\Multinomial}{Multinomial}
\DeclareMathOperator{\Cat}{Cat}
\DeclareMathOperator{\Dirichlet}{Dir}
\DeclareMathOperator{\PoissonGamma}{Poisson--Gamma}
\DeclareMathOperator*{\argmin}{argmin}

\title{On the Connection Between Non-negative Matrix Factorization and
Latent Dirichlet Allocation}

\author{%
  Benedikt Geiger \\
  Harvard University \\
  \texttt{benedikt\_geiger@hms.harvard.edu} \\
  \And
  Peter J. Park \\
  Harvard University \\
  \texttt{peter\_park@hms.harvard.edu}
}

\begin{document}

\maketitle

\begin{abstract}
Non-negative matrix factorization with the generalized Kullback--Leibler divergence (NMF) and latent Dirichlet allocation (LDA) are two popular approaches for dimensionality reduction of non-negative data. Here, we show that NMF with $\ell_1$ normalization constraints on the columns of both matrices of the decomposition and a Dirichlet prior on the columns of one matrix is equivalent to LDA. To show this, we demonstrate that explicitly accounting for the scaling ambiguity of NMF by adding $\ell_1$ normalization constraints to the optimization problem allows a joint update of both matrices in the widely used multiplicative updates (MU) algorithm. When both of the matrices are normalized, the joint MU algorithm leads to probabilistic latent semantic analysis (PLSA), which is LDA without a Dirichlet prior. Our approach of deriving joint updates for NMF also reveals that a Lasso penalty on one matrix together with an $\ell_1$ normalization constraint on the other matrix is insufficient to induce any sparsity.
\end{abstract}

\section{Introduction}

Non-negative matrix factorization (NMF) and latent Dirichlet allocation (LDA) are two prominent techniques for unsupervised machine learning and data analysis. Both approaches decompose non-negative data into an interpretable parts-based representation and have been used successfully across a wide range of areas including topic modeling, image and audio signal processing, recommendation systems, and bioinformatics \cite{jelodar2019latent,wang2012nonnegative,gillis2020nonnegative}. However, despite their conceptual similarities and widespread use, the relationship between NMF and LDA remains underexplored. Although partial results exist, prior work is written from a probabilistic perspective, and empirical comparisons often fall short to appreciate the depth of their connection.

Non-negative matrix factorization $X\approx WH$ with an appropriately chosen loss function is commonly viewed as a constrained optimization problem, with linear algebra algorithms employed to solve it. In the case of the widely used Kullback--Leibler (KL) divergence loss, the most popular approach for solving the optimization problem is the so-called multiplicative updates approach \cite{gillis2020nonnegative}. Originally derived by Richardson \cite{richardson1972bayesian} and Lucy \cite{lucy1974iterative} in the context of deconvolution, it was rediscovered and applied to NMF in the seminal work of Lee and Seung \cite{lee2000algorithms,lee1999learning}. The algorithm alternates between updating each of the matrices $W$ and $H$ of the factorization in a multiplicative manner. Latent Dirichlet allocation was introduced as a generative statistical model of text data to learn the underlying topics, topic proportions of documents, and probabilistic topic assignments of individual words \cite{blei2003latent}. As a bag-of-words model, LDA assumes that the order of the words in each document does not play a role, and only the word-document count matrix $X$ of a corpus is used during training and inference. The generative process of LDA extends that of its predecessor probabilistic latent semantic analysis (PLSA) \cite{hofmann2001unsupervised,hofmann1999probabilistic} by adding a Dirichlet prior to the topic proportions. Thus, a solution of PLSA can also be interpreted as a maximum \textit{a posteriori} estimate of LDA under a uniform Dirichlet prior \cite{girolami2003equivalence}.

In addition to these connections to LDA, PLSA is well-known to be closely related to NMF. For normalized inputs $X$, fixed points of the PLSA algorithm can be mapped to fixed points of the MU algorithm of NMF with the KL divergence, and \textit{vice versa} \cite{gaussier2005relation}. The identical result has also been shown for global maxima of the PLSA likelihood and global minima of the KL loss of NMF \cite{ding2008equivalence}. Noticing similarities between the update rules and underlying generative processes of NMF, PLSA, and LDA dates back to earlier work on unified frameworks for probabilistic matrix decomposition \cite{buntine2002variational,buntine2005discrete}. Various probabilistic latent variable models with joint parameter updates closely resembling the update rules of NMF have also been studied in \cite{shashanka2007sparse,shashanka2008probabilistic,welling2008deterministic}. However, an important detail frequently overlooked in the literature is that PLSA \textit{jointly} updates the topic and topic proportion matrices, whereas the MU algorithm of NMF \textit{sequentially} updates one matrix given the other. 

In this paper, we reveal a deeper connection between NMF and topic models like PLSA and LDA. We solve NMF with additional normalization constraints, and show that this leads to algorithms resembling or identical to the PLSA algorithm. Our main contributions are the following:

\begin{enumerate}[leftmargin=*]
    \item For NMF with the KL divergence, adding normalization constraints to the columns of one or both matrices of the decomposition makes it possible to derive \textit{joint} update rules for both matrices. The obtained algorithms save one matrix multiplication per iteration. 
    \item NMF with the KL divergence and additional normalization constraints on the columns of both matrices is equivalent to PLSA. Both the objective function and the canonical joint update rules are identical. In contrast to the existing literature, we establish an algorithmic equivalence, and our result neither requires normalizing the input matrix $X$ nor transforming the trained model parameters.
    \item With a Dirichlet prior on the columns of the matrix corresponding to the topic proportions, NMF with the KL divergence and additional normalization constraints on the columns of both matrices is equivalent to LDA. Both the objective function and the canonical variational inference algorithm are identical.
    \item Joint update equations can be derived to speed up and better understand variations of NMF. We demonstrate this by revisiting sparse NMF with the KL divergence \cite{liu2003non}. 
\end{enumerate}

Joint update rules for NMF have recently been introduced in the more general case of $\beta$-divergence loss functions \cite{marmin2023joint}. Both in \cite{marmin2023joint} and in our work, the results are based on a majorization-minimization scheme in which an upper bound of the objective function is constructed and minimized in every iteration \cite{sun2016majorization}. The approach differs from previous work by jointly considering both matrices of the decomposition instead of constructing an upper bound for one of the two matrices at a time \cite{fevotte2011algorithms,nakano2010convergence}. In contrast to \cite{marmin2023joint}, however, the additional normalizations allow us to derive a closed-form expression for the optimum of the upper bound.

That NMF with Dirichlet priors on the columns of \textit{both} matrices of the decomposition leads to LDA with Dirichlet smoothing on the topics was briefly mentioned in \cite{zhou2012beta,zhou2013negative} and suggested in \cite[Section 3.5]{cemgil2019bayesian}. By revisiting the connection between PLSA and NMF, our perspective emphasizes the role of the normalization constraints rather than the Dirichlet distribution giving rise to the connections between generative models and algorithms.
\section{Related work}

\paragraph{Bayesian NMF.}
Non-negative matrix factorization can often be interpreted in a statistical framework \cite{fevotte2009nonnegative}. For example, minimizing the KL divergence between the data matrix $X$ and its reconstruction $WH$ is equivalent to maximizing the likelihood of $W$ and $H$ in an element-wise Poisson noise model $X\sim\Poisson(WH)$. Similar connections between the loss function and the likelihood of an appropriately chosen generative model also exist for alternative reconstruction errors \cite{tan2012automatic}. In other words, the probabilistic point of view on NMF is based on the underlying statistical assumptions of the chosen reconstruction error. Through priors on $W$ or $H$, this statistical framework allows one to induce desired properties of the factorization. Bayesian extensions of NMF have, for instance, been applied to control the sparsity or smoothness of the parameters \cite{cemgil2009bayesian,virtanen2008bayesian,fevotte2009isdivergence}, automatically select the factorization rank \cite{tan2012automatic,hoffman2010bayesian}, obtain uncertainty estimates \cite{schmidt2009bayesian}, or model bounded support data \cite{ma2014variational}.

In the case of NMF with the KL divergence, the most commonly used prior is the Gamma distribution \cite{canny2004gap,virtanen2008bayesian,cemgil2009bayesian,fevotte2009nonnegative,dikmen2012maximum,mohammadiha2013supervised,paisley2014bayesian}. The main reason for this choice is the analytical convenience resulting from the Gamma distribution being the conjugate prior of the Poisson distribution. The obtained Gamma--Poisson model with fixed prior parameters has been shown to be related to LDA: When $W$ is column-normalized and Gamma priors with identical rate parameters are placed on the columns of $H$, the generative model extends LDA by also modeling the document length \cite[Lemma 1]{buntine2005discrete}. In fact, we strengthen this connection and prove that the canonical variational inference algorithm of this Gamma--Poisson model is identical to the LDA algorithm from \cite{blei2003latent} (see \Cref{appendix:gamma_poisson}). This algorithmic equivalence can be understood as the Bayesian counterpart to the connection between NMF with a normalization constraint on $W$ and PLSA.
\section{Preliminaries}\label{section_preliminaries}

Throughout the paper, our terminology draws inspiration from the topic modeling literature. That is, when considering the non-negative matrix decomposition $X\approx WH$, we use the terms (unnormalized) topic matrix to refer to $W$ and topic weights to refer to $H$. When additional normalization constraints are in place, we use (normalized) topic matrix and topic proportions instead. We use the notation $\R_+$ to refer to the non-negative real numbers, and for any positive integer $M$, we denote the $M-1$ dimensional standard simplex by $\Delta_{M-1}=\{y\in\R_+^M \, | \, \sum_m y_m = 1\}$. When using the abbreviation NMF, we refer to non-negative matrix factorization with the KL divergence unless noted otherwise.

In the context of topic modelling, the goal of NMF, PLSA, and LDA is to identify the underlying topics in a corpus of documents and to represent each document as an additive combination of the learned topics. All three methods are bag-of-words models---they assume that learning the decomposition does not require taking the order of words in the documents into account. Therefore, the input of the algorithms is the term-document matrix $X\in\R^{V\times D}_+$, where each entry $x_{vd}$ is the number of times the term $v$ occurs in document $d$. The number of topics $K$ is a hyperparameter of all three approaches.

\subsection{Non-negative matrix factorization (NMF)}
NMF decomposes the term-document matrix $X$ into two non-negative matrices $W\in\R_+^{V\times K}$ and $H\in\R_+^{K\times D}$. The columns of $W$ and $H$ can then be interpreted as the unnormalized topics and topic weights of the documents, respectively, and these parameters are learned by minimizing an appropriate ``distance`` $D(X\,|| \,WH)$ between the data matrix and its reconstruction. A common choice is the (generalized) Kullback--Leibler divergence
\begin{align}\label{equation_kl_divergence}
    \KL(X\, || \, WH) = \sum_{v,d} x_{vd}\log\frac{x_{vd}}{(WH)_{vd}}- x_{vd} + (WH)_{vd},
\end{align}
which leads to the optimization problem
\begin{align}\label{optimization_problem_nmf}
    \min_{W\in\R_{+}^{V\times K}, \, H\in\R_{+}^{K\times D}} \KL(X\, || \,WH).
\end{align}
Algorithms to solve \eqref{optimization_problem_nmf} typically alternate between updating each of the matrices $W$ and $H$ while fixing the other one \cite{hien2021algorithms}. The motivation for this approach is two-fold \cite{gillis2020nonnegative}. First, the factorization $X\approx WH$ is symmetric to $X^T \approx H^T W^T$. An update rule for one of the matrices thus immediately yields an update rule for the other. Second, the optimization problems induced by fixing each of the two matrices in \eqref{optimization_problem_nmf} are convex. The most common algorithm consists of multiplicatively updating each matrix iteratively until convergence (see \Cref{alg:mu_nmf}). It was popularized by Lee and Seung, who demonstrated the potential of learning parts-based representations and proved the monotonic convergence of the objective function \cite{lee1999learning,lee2000algorithms,wang2012nonnegative}. The key idea of their derivation is to construct a well-behaved function $G(H, H')$ for $F(H) = \KL(X\, || \,WH)$ such that
\begin{equation}
    \begin{split}
        G(H, H') &\geq F(H) \quad \text{for all } H, \, H', \\
        G(H', H') &= F(H') \quad \text{for all } H'.
    \end{split}
\end{equation}
Once such an \textit{auxiliary function} $G$ is found, the objective function $F$ can be shown to be non-increasing by updating an iterate $H^{(n)}$ to the argument of the minimum of $H \mapsto G(H, H^{(n)})$ (see \Cref{appendix:auxiliary_functions}). While the perspective on \eqref{equation_kl_divergence} and the construction of the auxiliary function in \cite{lee2000algorithms} via Jensen's inequality is entirely non-probabilistic, there is also a probabilistic framework for \Cref{alg:mu_nmf} \cite{cemgil2009bayesian,fevotte2009nonnegative}. Assuming conditional independence of the counts $X_{vd}$, the model $X_{vd} \sim \Poisson(WH)_{vd}$ leads to a log-likelihood identical to the negative KL divergence \eqref{equation_kl_divergence} (up to a constant). By augmenting the generative process to a latent variable model via
\begin{equation}\label{gen_process_nmf}
        z_{vkd} \sim \Poisson(w_{vk}h_{kd}), \quad X_{vd} = \sum_k z_{vkd},
\end{equation}
a standard expectation-maximization (EM) \cite{dempster1977maximum} procedure yields a \textit{joint auxiliary function} $G$ for the model parameters $(W,H)$, i.e.,
\begin{equation}
    \begin{split}
        G((W, H), (W', H')) &\geq \KL(X\, || \, WH) \quad \text{for all } (W,H), \, (W',H'), \\
        G((W',H'), (W',H')) &= \KL(X\, || \, W'H') \quad \text{for all } (W', H').
    \end{split}
\end{equation}
With
\begin{equation}
    \phi_{vkd}^{(n)} \coloneqq \frac{w_{vk}^{(n)} h_{kd}^{(n)}}{(W^{(n)}H^{(n)})_{vd}}
\end{equation}
and modulo terms independent of the model parameters, the auxiliary function is given by
\begin{equation}\label{nmf_joint_auxiliary_function}
    G((W, H), (W^{(n)}, H^{(n)})) = -\sum_{v,k,d} x_{vd} \phi_{vkd}^{(n)} \log \frac{w_{vk}h_{kd}}{\phi_{vkd}^{(n)}} + \sum_{v,d} (WH)_{vd},
\end{equation}
cf.\ \cite[eq.\ (18)]{cemgil2009bayesian}. Although a closed-form joint optimum does not exist, $G$ induces auxiliary functions for both individual matrices via
\begin{equation}\label{auxiliary_function_separate}
    \begin{split}
        \widetilde{G}(W, W^{(n)}) &= G((W, H^{(n)}), (W^{(n)}, H^{(n)})), \\
    \widetilde{G}(H, H^{(n)}) &= G((W^{(n)}, H), (W^{(n)}, H^{(n)})).
    \end{split}
\end{equation}
Alternating between optimizing each of the induced auxiliary functions recovers the multiplicative updates from Lee and Seung, revealing that \Cref{alg:mu_nmf} is in fact a generalized EM algorithm.

In practice, it is often desirable to increase the interpretability of the matrices and to eliminate the scaling ambiguity of the factorization caused by $WH = WDD^{-1}H$ for any invertible diagonal matrix $D\in\R^{K\times K}_{+}$. A common strategy to achieve both is to normalize the columns of $W$ to probability distributions and to scale the rows of $H$ accordingly after each iteration or after convergence.

\begin{algorithm}
    \caption{\cite[Theorem 2]{lee2000algorithms} Multiplicative updates NMF}\label{alg:mu_nmf}
    \begin{algorithmic}
        \Require $X\in\R_{+}^{V\times D}$, initializations $W^{(0)}\in\R_{+}^{V\times K}, \, H^{(0)}\in\R_{+}^{K\times D}$
        \State $n \gets 0$
        \While {$\KL(X\, || \,W^{(n)}H^{(n)})$ not converged}
        \State $w^{(n+1)}_{vk} = w^{(n)}_{vk} \frac{\sum_d x_{vd} \frac{ h^{(n)}_{kd} }{ (W^{(n)}H^{(n)})_{vd} } }{ \sum_d h^{(n)}_{kd} }$
        \State $h^{(n+1)}_{kd} = h^{(n)}_{kd} \frac{\sum_v x_{vd} \frac{ w^{(n+1)}_{vk} }{ (W^{(n+1)}H^{(n)})_{vd} } }{ \sum_v w^{(n+1)}_{vk} }$
        \State $n \gets n + 1$
        \EndWhile
    \end{algorithmic}
\end{algorithm}

\subsection{Probabilistic latent semantic analysis (PLSA)}
In PLSA \cite{hofmann2001unsupervised}, the generative process of the term-document matrix is unfolded in the sense that the words are generated individually. The scheme for each of the $n=1,\ldots,N_d$ words in document $d$ is the following:
\begin{enumerate}
    \item Sample a latent topic $z_{nd}$ with probability $p(z_{nd}=k) = h_{kd}$.
    \item Sample a word $\nu_{nd}$ with probability $p(\nu_{nd}=v|z_{nd}=k) = w_{vk}$.
\end{enumerate}
The resulting log-likelihood is given by
\begin{equation}\label{plsa_loglikelihood}
    \log p(X | W, H) = \sum_{v,d} x_{vd} \log(WH)_{vd} + c(X),
\end{equation}
where the constant $c(X)$ arises from bagging the sequence of words, cf.\ \cite[Section \rom{3}.A]{buntine2005discrete}. In the following, we will always ignore this constant because it only depends on the observed counts and does not affect model training. The relevant part of the log-likelihood \eqref{plsa_loglikelihood} therefore reduces to
\begin{equation}\label{objective_plsa}
    \mathcal{L}^{\text{PLSA}}(X | W, H) \coloneqq \sum_{v,d} x_{vd} \log(WH)_{vd}.
\end{equation}
The algorithm to fit the model parameters is an EM algorithm, described in \cite{hofmann2001unsupervised}. With the conditional distributions of the latent topic assignments for model parameters $(W^{(n)}, H^{(n)})$ equal to
\begin{align}
    p(z_{nd}=k|\nu_{nd}=v)^{(n)} = \phi_{vkd}^{(n)} \propto w_{vk}^{(n)}h_{kd}^{(n)},
\end{align}
the E-step of the derivation naturally leads to an auxiliary function of the negative objective function \eqref{objective_plsa} given by
\begin{equation}\label{plsa_joint_auxiliary_function}
    G((W, H), (W^{(n)}, H^{(n)})) = -\sum_{v,k,d} x_{vd} \phi_{vkd}^{(n)} \log \frac{w_{vk}h_{kd}}{\phi_{vkd}^{(n)}}.
\end{equation}
The minimization of $G$ during the M-step yields \Cref{alg:plsa}. While the algorithm is remarkably similar to \Cref{alg:mu_nmf} and the only difference seems to be the normalizations, it is important to note that all model parameters are updated jointly in PLSA. The conditional distributions of the topic assignments $\phi_{vkd}$ are not recomputed after updating one of the matrices.

\begin{algorithm}
    \caption{\cite[eq.\ (11), (12)]{hofmann2001unsupervised} EM algorithm PLSA}\label{alg:plsa}
    \begin{algorithmic}
        \Require $X\in\R_{+}^{V\times D}$, initializations $W^{(0)}\in\Delta_{V-1}^K, \, H^{(0)}\in\Delta_{K-1}^D$
        \State $n \gets 0$
        \While {$\mathcal{L}^{\text{PLSA}}(X | W^{(n)}, H^{(n)})$ not converged}
        \State $\phi_{vkd}^{(n)} \propto w_{vk}^{(n)} h_{kd}^{(n)}$
        \State $w^{(n+1)}_{vk} \propto \sum_d x_{vd} \phi_{vkd}^{(n)}$
        \State $h^{(n+1)}_{kd} \propto \sum_v x_{vd} \phi_{vkd}^{(n)}$
        \State $n \gets n + 1$
        \EndWhile
    \end{algorithmic}
\end{algorithm}

\subsection{Latent Dirichlet allocation (LDA)}
LDA was developed to overcome the issue that the number of model parameters of PLSA grows linearly with the number of documents on which the model is trained \cite{blei2003latent,blei_introduction_2012}. To this end, LDA introduces a model parameter $\alpha\in\R^K_{>0}$ and a Dirichlet prior $h \sim \Dirichlet(\alpha)$ that turn the topic proportions into latent variables. The generative process for each document $d$ then becomes
\begin{enumerate}
    \item Sample the topic proportions $h_d \sim \Dirichlet(\alpha).$
    \item For each of the $n=1,\ldots,N_d$ words in document $d$:
    \begin{enumerate}
        \item Sample a latent topic $z_{nd}$ with probability $p(z_{nd}=k) = h_{kd}$,
        \item Sample a word $\nu_{nd}$ with probability $p(\nu_{nd}=v|z_{nd}=k) = w_{vk}$.
    \end{enumerate}
\end{enumerate}
This results in the log-likelihood
\begin{align}
    \log p(x_d | W, \alpha) &= \log \int p(h | \alpha) \prod_v (Wh)_v^{x_{vd}} \, dh,
\end{align}
where $p(h | \alpha)$ denotes the density of the Dirichlet distribution. The introduction of the Dirichlet prior causes the posterior distribution of the latent variables $(h_d, z_d)$ to be intractable. Blei et al.\ therefore resort to variational inference (VI), which is one of the approximate inference algorithms applicable to the LDA model \cite{jordan1999introduction,blei2003latent,asuncion2012smoothing}. Following their mean-field approach, the true posterior is approximated with a variational distribution $q$ given by
\begin{align}
    q(h, z | \beta, \phi) = \prod_d q_d(h_d, z_d | \beta_d, \phi_d), \qquad q_d(h_d, z_d | \beta_d, \phi_d) = q(h_d|\beta_d) \prod_n q(z_{nd}|\phi_{nd}),
\end{align}
where $\beta_d$ and $\phi_{nd}$ are parameters of a Dirichlet and categorical distribution, respectively. In addition to the model parameters $W$ and $\alpha$, these so-called variational parameters are iteratively optimized by the algorithm. Instead of maximizing the exact log-likelihood, the idea of variational inference is to optimize a lower bound given by
\begin{equation}\label{lda_elbo}
    \mathcal{L}^{\text{LDA}}(X |W, \alpha, \beta, \phi) = \Eq{\log p(X, h, z | W, \alpha)} - \Eq{\log q(h, z | \beta, \phi)}.
\end{equation}
Here and in the following, $\Eq{\cdot}$ denotes the expectation with respect to the variational distribution. Similar to PLSA, the key insight that reveals algebraic connections to NMF is that the optimal variational approximation of the posterior of the topic assignments $z_{nd}$ does not depend on the position $n$, but only on the word $\nu_{nd}$ at position $n$, i.e., $\phi_{nd}$ can be reindexed. Together with
\begin{equation}
    \widetilde{h}_{kd}(\beta_d) \coloneqq \exp\big(\mathbb{E}_{q(h_d|\beta_d)}[\log h_{kd}]\big),
\end{equation}
which can be computed explicitly using the digamma function, the variational lower bound \eqref{lda_elbo} reduces to
\begin{align}\label{lda_elbo_concrete}
    \begin{split}
        \mathcal{L}^{\text{LDA}}(X | W, \alpha, \beta, \phi) &= \sum_{v,k,d} x_{vd} \phi_{vkd} \log \frac{w_{vk}\widetilde{h}_{kd}}{\phi_{vkd}} + \sum_d \Big( \log\Gamma(\textstyle{\sum}_k \alpha_k) - \log\Gamma(\textstyle{\sum}_k \beta_{kd}) \Big) \\
        &\phantom{=} + \displaystyle\sum_{k,d} \Big( \log \Gamma(\beta_{kd}) - \log\Gamma(\alpha_k) + (\alpha_k - \beta_{kd}) \Eq{\log h_{kd}} \Big),
    \end{split}
\end{align}
cf.\ \cite[A.1, eq.\ (15), (16)]{blei2003latent}. The VI algorithm is a variational EM procedure. It alternates between optimizing the approximation of the posterior distribution of the latent variables for fixed model parameters (variational E-step) and optimizing the model parameters for a fixed variational distribution (variational M-step). That is, \eqref{lda_elbo_concrete} is optimized with respect to $(\beta, \phi)$ during the variational E-step and with respect to $(W, \alpha)$ during the variational M-step. Since there is no closed-form optimum for the parameter of the Dirichlet prior $\alpha$, we assume that it is fixed for simplicity. This procedure results in \Cref{alg:lda}.

\begin{algorithm}
    \caption{\cite[A.3, A.4]{blei2003latent} VI algorithm LDA}\label{alg:lda}
    \begin{algorithmic}
        \Require $X\in\R_{+}^{V\times D}, \, \alpha\in\R_{>0}^K$, initializations $W^{(0)}\in\Delta_{V-1}^K, \, \beta^{(0)}\in\R_{>0}^{K \times D}$
        \State $n \gets 0$
        \While {$\mathcal{L}^{\text{LDA}}(X | W, \alpha, \beta, \phi)$ not converged}
        \State $\widetilde{h}_{kd}^{(n)} = \exp(\Eq{\log h_{kd}})$
        \State $\phi_{vkd}^{(n)} \propto w_{vk}^{(n)} \widetilde{h}_{kd}^{(n)}$
        \State $w^{(n+1)}_{vk} \propto \sum_d x_{vd} \phi_{vkd}^{(n)}$
        \State $\beta^{(n+1)}_{kd} = \alpha_k + \sum_v x_{vd} \phi_{vkd}^{(n)}$
        \State $n \gets n + 1$
        \EndWhile
    \end{algorithmic}
\end{algorithm}
\section{Algorithms for NMF with normalization constraints}\label{section_algorithms}

We now consider the NMF optimization problem \eqref{optimization_problem_nmf} with additional normalization constraints on the topics or both the topics and the topic weights. We describe the connections between the original formulation and the sum-constrained optimization problems, and derive NMF algorithms with joint update rules. These new NMF algorithms are based on the observation that a constrained optimization of the joint auxiliary function \eqref{nmf_joint_auxiliary_function} is possible. 

More precisely, the two optimization problems we consider are NMF with normalized topics
\begin{align}\label{optimization_problem_nmf_constraint1}
    \min_{W\in\Delta_{V-1}^K, \, H\in\R_+^{K\times D}} D_{\text{KL}}(X || WH),
\end{align}
as well as NMF with both normalized topics and topic proportions
\begin{align}\label{optimization_problem_nmf_constraint2}
    \min_{W\in\Delta_{V-1}^K, \, H\in\Delta_{K-1}^D} D_{\text{KL}}(X || WH).
\end{align}
At first glance, it might seem unnatural to normalize both matrices $W$ and $H$ without adjusting the input matrix $X$. However, with both constraints, we have $\sum_{v} (WH)_{vd}=1$, and the objective function of \eqref{optimization_problem_nmf_constraint2} becomes
\begin{equation}
    -\sum_{v,d}x_{vd} \log(WH)_{vd} = - \sum_d \Big(\sum_{v'} x_{v'd}\Big) \sum_{v} \frac{x_{vd}}{\sum_{v'} x_{v'd}} \log(Wh_d)_v.
\end{equation}
The inner sum corresponds to the Kullback--Leibler divergence between the empirical distribution induced by the observed counts $(x_{vd})_v$ and the model distribution $(Wh_d)_v$ \cite[eq.\ (10)]{hofmann1999probabilistic}. In particular, optimization problem \eqref{optimization_problem_nmf_constraint2} and PLSA have the same objective function. Both models project the samples onto the simplex spanned by the topics based on the KL divergence. The following lemmas clarify the connections between all three NMF problems with different normalization constraints.

\begin{lemma}\label{lemma_connection_nmf_nmfc1}
    Let $(W, H)$ be a solution of the standard NMF optimization problem \eqref{optimization_problem_nmf} and let $\lambda_k = \sum_v w_{vk}$. Then $(\widetilde{W}, \widetilde{H})$ with $\widetilde{w}_{vk} = w_{vk} / \lambda_k$ and $\widetilde{h}_{kd} = \lambda_k h_{kd}$ is a solution of NMF with a normalization constraint on $W$ \eqref{optimization_problem_nmf_constraint1}. Conversely, every solution of \eqref{optimization_problem_nmf_constraint1} is a solution of \eqref{optimization_problem_nmf}.
\end{lemma}
\begin{lemma}\label{lemma_connection_nmf_nmfc2}
    Let $(W, H)$ be a solution of NMF with a normalization constraint on $W$ \eqref{optimization_problem_nmf_constraint1} and let $\lambda_d = \sum_v x_{vd}$. Then $(W, \widetilde{H})$ with $\widetilde{h}_{kd} = h_{kd} / \lambda_d$ is a solution of NMF with a normalization constraint on both $W$ and $H$ \eqref{optimization_problem_nmf_constraint2}.
    Conversely, let $(W, H)$ be a solution of \eqref{optimization_problem_nmf_constraint2}. Then $(W, \widetilde{H})$ with $\widetilde{h}_{kd} = \lambda_d h_{kd}$ is a solution of \eqref{optimization_problem_nmf_constraint1}.
\end{lemma}
\Cref{lemma_connection_nmf_nmfc1} follows directly from $WH = WDD^{-1}H$ for any invertible non-negative diagonal matrix $D$ and the proof of \Cref{lemma_connection_nmf_nmfc2} is deferred to \Cref{appendix:nmf_with_norm_constraints}. Taken together, these two results show that a solution of one of the three NMF problems \eqref{optimization_problem_nmf}, \eqref{optimization_problem_nmf_constraint1} or \eqref{optimization_problem_nmf_constraint2} can be readily converted into a solution of one of the other optimization problems. In fact, one can also see that an analogous result holds for the fixed points of their MU algorithms (see below).

Contrary to intuition, the additional normalization constraints simplify the NMF optimization problem. Recall that the joint auxiliary function \eqref{nmf_joint_auxiliary_function} of $\KL(X\,||\,WH)$ is given by
\begin{align*}
    G((W, H), (W^{(n)}, H^{(n)})) &= -\sum_{v,k,d} x_{vd} \phi_{vkd}^{(n)} \log \frac{w_{vk}h_{kd}}{\phi_{vkd}^{(n)}} + \sum_{v,d} (WH)_{vd}.
\end{align*}
Under the constraints $w_k\in\Delta_{V-1}$, the second summand reduces to $\sum_{k,d}h_{kd}$ and $G$ can be jointly optimized in all model parameters $(W, H)$. Using Lagrange multipliers, we obtain the joint multiplicative update equations given in \Cref{alg:mu_nmf_constraint1} (see \Cref{appendix:nmf_with_norm_constraints} for details). When additionally imposing $h_d\in\Delta_{K-1}$, the relevant summands of the auxiliary function are identical to the auxiliary function \eqref{plsa_joint_auxiliary_function} of PLSA. Accordingly, the canonical multiplicative algorithm for \eqref{optimization_problem_nmf_constraint2} is the PLSA algorithm, cf.\ \Cref{alg:mu_nmf_constraint2,alg:plsa}. The two models are equivalent both in the sense that their optimization problems are the same and in the sense that their MU algorithms coincide.

Although it is also possible to derive update rules that alternate between updating $W$ and $H$ using the induced auxiliary functions \eqref{auxiliary_function_separate}, the joint update equations have one major advantage. They do not require $WH$ to be recomputed after updating one of the two matrices, which reduces the number of matrix multiplication per iteration from four to three. From a theoretical perspective, the joint updates turn NMF from a generalized EM algorithm into an EM algorithm. Another interesting observation is that the generative process \eqref{gen_process_nmf} of NMF does not explicitly model the latent topic assignments of individual words in documents. However, the correspondence to PLSA and its generative process shows that $\phi_{vkd}$ is in fact the posterior probability that a word $v$ in document $d$ has been generated by topic $k$. For given data $X$, the two generative processes of NMF with normalization constraints and PLSA are equivalent and they reveal different perspectives on the factorization. On the one hand, the Poisson noise model of NMF is more concise and motivates to not store $\phi_{vkd}$ explicitly during model training. On the other hand, the atomic generative process of PLSA adds a layer of interpretability to the model parameters. More details on the connections between various generative models are given in \Cref{appendix:equivalent_gen_models}, and we also refer the interested reader to \cite{buntine2005discrete}.

\begin{algorithm}
    \caption{Joint multiplicative updates for NMF \eqref{optimization_problem_nmf_constraint1}}\label{alg:mu_nmf_constraint1}
    \begin{algorithmic}
        \Require $X\in\R_{+}^{V\times D}$, initializations $W^{(0)}\in\Delta_{V-1}^K, \, H^{(0)}\in\R_{+}^{K\times D}$
        \State $n \gets 0$
        \While {$\KL(X \, || \, W^{(n)}H^{(n)})$ not converged}
        \State $w^{(n+1)}_{vk} \propto w^{(n)}_{vk} \sum_d x_{vd} \frac{ h^{(n)}_{kd} }{ (W^{(n)}H^{(n)})_{vd} }$
        \State $h^{(n+1)}_{kd} = h^{(n)}_{kd} \sum_v x_{vd} \frac{ w^{(n)}_{vk} }{ (W^{(n)}H^{(n)})_{vd} }$
        \State $n \gets n + 1$
        \EndWhile
    \end{algorithmic}
\end{algorithm}
\section{Connection between NMF and LDA}\label{section_connections_lda}

The previous section revealed the equivalence of NMF with additional normalization constraints and PLSA. Since the only difference between PLSA and LDA is the Dirichlet prior on the topic proportions, it is natural to wonder whether NMF with normalization constraints and a Dirichlet prior on the columns of $H$ is equivalent to LDA. This turns out to be the case.

With the normalization constraints $w_k \in\Delta_{V-1}$ on the topic matrix $W$, the following Dirichlet--Poisson generative model yields, up to a data-dependent constant, the same log-likelihood and variational lower bound as LDA:
\begin{enumerate}
    \item Sample the topic proportions $h_d \sim \Dirichlet(\alpha)$.
    \item Sample the topic contributions $z_{vkd} \sim \Poisson(w_{vk}h_{kd})$.
    \item $X_{vd} = \sum_k z_{vkd}$.
\end{enumerate}
Denoting the slice $z_{::d}$ of the latent tensor $z$ as $z_d$ and the fiber $z_{v:d}$ as $z_{vd}$, we apply a mean-field variational inference approach and approximate the true posterior of the latent variables with a variational distribution given by
\begin{align}
    q(h, z | \beta, \phi) = \prod_d q_d(h_d, z_d | \beta_d, \phi_d), \qquad q_d(h_d, z_d | \beta_d, \phi_d) = q(h_d | \beta_d) \prod_v q(z_{vd} | x_{vd}, \phi_{vd}),
\end{align}
where $\beta_d$ and $\phi_{vd}$ are Dirichlet and multinomial parameters, respectively. With the joint likelihood of the observed and latent variables of a single document being equal to
\begin{equation}
    p(x_d, h_d, z_d | W, \alpha) = \Big(\prod_v \mathbbm{1}_{\sum_k z_{vkd} = x_{vd}}\Big) p(h_d | \alpha) \prod_{v,k} p(z_{vkd} | W, h_d),
\end{equation}
the relevant terms of the variational lower bound of the Dirichlet--Poisson model with normalized topic matrix $W$ are identical to the variational lower bound \eqref{lda_elbo_concrete} of LDA, and the VI algorithms are identical too (see \Cref{appendix:dirichlet_poisson}). Similar to the connection between NMF with normalization constraints and PLSA, the two different generative processes reveal different perspectives on the factorization. On the one hand, the Dirichlet--Poisson model of NMF is more concice and motivates to not store the variational parameters $\phi$ explicitly during model training---a computational trick that is already used in efficient LDA implementations \cite{hoffman2010online,vrehuuvrek2010software}, cf.\ \Cref{alg:dpnmf,alg:lda}. 
On the other hand, the multinomial generative process of LDA adds a layer of interpretability to $\phi$.

\begin{algorithm}[H]
    \caption{VI algorithm Dirichlet--Poisson NMF / LDA}\label{alg:dpnmf}
    \begin{algorithmic}
        \Require $X\in\R_{+}^{V\times D}, \, \alpha\in\R_{>0}^K$, initializations $W^{(0)}\in\Delta_{V-1}^K, \, \beta^{(0)}\in\R_{>0}^{K \times D}$
        \State $n \gets 0$
        \While {$\mathcal{L}^{\text{LDA}}(X | W, \alpha, \beta, \phi)$ not converged}
        \State $\widetilde{h}_{kd}^{(n)} = \exp(\Eq{\log h_{kd}})$
        \State $w^{(n+1)}_{vk} \propto w^{(n)}_{vk} \sum_d x_{vd} \frac{\widetilde{h}_{kd}^{(n)}}{(W^{(n)}\widetilde{H}^{(n)})_{vd}}$
        \State $\beta^{(n+1)}_{kd} = \alpha_k + \widetilde{h}_{kd}^{(n)} \sum_v x_{vd} \frac{w_{vk}^{(n)}}{(W^{(n)}\widetilde{H}^{(n)})_{vd}}$
        \State $n \gets n + 1$
        \EndWhile
    \end{algorithmic}
\end{algorithm}

\section{Sparse NMF}\label{section_applications}

In sparse NMF, a penalty on $H$ is added to decrease the number of topics used in the reconstruction of each document. This is especially useful when the number of topics is high, as it can prevent overfitting and improve the interpretability of $W$ \cite{le2015sparse}. For a hyperparameter $\lambda > 0$, the objective function of $\ell_1$-penalized NMF is given by
\begin{equation}
    \KL(X \, || \, WH) + \lambda \|H\|_1.
\end{equation}
An important technical difficulty of sparse NMF is that the scaling ambiguity of the matrix product $WH$ in combination with a penalty on $H$ would cause $H$ to converge to zero and $W$ to blow up unless further constraints are imposed. For this reason, the columns of $W$ are usually required to be normalized, and the optimization problem becomes
\begin{align}\label{optimization_problem_nmf_sparse_l1}
    \min_{W\in\Delta_{V-1}^K, \, H\in\R_{+}^{K\times D}} \KL(X \, || \, WH) + \lambda \|H\|_1.
\end{align}
With the $\ell_1$ normalization constraints on the columns of $W$ in place, the results from \Cref{section_algorithms} can be readily applied to obtain joint update rules to solve \eqref{optimization_problem_nmf_sparse_l1}. Similar to before, the objective function value can be shown to be non-increasing under the following joint updates:
\begin{align}
    w^{(n+1)}_{vk} \propto w^{(n)}_{vk} \sum_d x_{vd} \frac{ h^{(n)}_{kd} }{ (W^{(n)}H^{(n)})_{vd} }, \qquad h^{(n+1)}_{kd} = \frac{1}{1+\lambda} h^{(n)}_{kd} \sum_v x_{vd} \frac{ w^{(n)}_{vk} }{ (W^{(n)}H^{(n)})_{vd} }.
\end{align}
The only difference between $\ell_1$-penalized NMF and NMF without a penalty \eqref{optimization_problem_nmf_constraint1} is that the topic weights $H$ are scaled by a factor of $\frac{1}{1+\lambda}$, cf.\ \Cref{alg:mu_nmf_constraint1}. Surprisingly, however, this holds true not only for the first, but for \textit{every} iteration of the algorithm. Our observation is not a consequence of the algorithmic approach, but a property of the optimization problem itself:

\begin{lemma}\label{lemma_lasso_insufficient}
    Let $(W, H)$ be a solution of NMF with an $\ell_1$ normalization constraint on $W$ \eqref{optimization_problem_nmf_constraint1} and let $\lambda > 0$. Then $(W, \tfrac{1}{1+\lambda}H)$ is a solution of Lasso-penalized NMF with an $\ell_1$ normalization constraint \eqref{optimization_problem_nmf_sparse_l1}. Conversely, let $(W, H)$ be a solution of \eqref{optimization_problem_nmf_sparse_l1}. Then $(W, (1+\lambda)H)$ is a solution of \eqref{optimization_problem_nmf_constraint1}.
\end{lemma}
A proof by contradiction is provided in \Cref{appendix:sparse_nmf}. The Lasso penalty being insufficient to induce sparse topic weights $H$ has, to the best of our knowledge, not been reported in the literature \cite{marmin2023majorization}. Interestingly, however, replacing the $\ell_1$ normalization constraints on the columns of $W$ with $\ell_2$ normalization constraints and using the heuristic gradient-based algorithm from \cite{le2015sparse} leads to sparsity \cite[Section \rom{4}]{leplat2020blind}. In fact, any normalization constraints on the columns of $W$ and penalty terms on $H$ can be proven to act as competing regularizers on the two matrices (see \Cref{appendix:sparse_nmf}). Together with \Cref{lemma_lasso_insufficient}, this demonstrates that the effect on the sparsity of the model parameters not only depends on the penalty on $H$, but also on the choice of the normalization constraint on $W$ and the reconstruction error.
\section{Conclusion}\label{section:conclusion}

This paper has examined NMF with normalization constraints on the columns of one or both matrices of the decomposition. By taking these constraints into account explicitly during the majorization-minimization scheme, it became possible to derive joint update rules that speed up model training significantly. The joint update rules further bridge the gap between NMF optimization problems and algorithms from the topic modeling literature such as PLSA and LDA. Leveraging these connections, we have demonstrated that different latent structures of equivalent generative models can highlight complementary aspects of the same model.

The existing theoretical and applied research on other aspects and extensions of NMF or LDA is beyond the scope of our work. However, we anticipate that the outlined connections are not limited to the classical full-batch variational inference algorithms of these base models. Regarding NMF, directly resolving the scaling ambiguity in the case of the Kullback--Leibler divergence allowed us to re-examine the Lasso penalty for sparse NMF and might also be helpful in answering questions on the convergence of the model parameters \cite{badeau2010stability}. An interesting but ambitious direction would be to extend these ideas to more general loss functions.

\section*{Acknowledgments}
This work was funded by the National Institutes of Health (grant number R01CA269805). All authors declare no conflict of interest.

We would like to express our gratitude to Marinka Zitnik for her advice throughout the project. Moreover, we thank Christian Fiedler, Teng Gao, Allen Lynch, and Dominik Glodzik for their helpful feedback on the manuscript.

{
\small

}

\newpage
\appendix

\section{Auxiliary functions for NMF}\label{appendix:auxiliary_functions}
We give a brief overview of the fundamental idea of majorization-minimization algorithms and the construction of auxiliary functions for NMF with the KL divergence. This section is based on \cite{sun2016majorization,lee2000algorithms,gillis2020nonnegative}.

Consider the optimization problem
\begin{equation*}
    \min_{Y \in \mathcal{Y}} F(Y)
\end{equation*}
for $\mathcal{Y} \subseteq \R^M$ and a function $F\colon \mathcal{Y} \rightarrow \R.$ Within the context of the majorization-minimization framework, the strategy to minimize $F$ consists of two-steps. First, an auxiliary function $G\colon \mathcal{Y} \times \mathcal{Y} \rightarrow \R$ with
\begin{align*}
    G(Y, Y') \geq F(Y), \\
    G(Y', Y') = F(Y')
\end{align*}
for all $Y, Y' \in \mathcal{Y}$ is constructed. The majorizer $G$ of $F$ is then minimized in every iteration, that is,
\begin{equation*}
    Y^{(n+1)} = \argmin_{Y \in\mathcal{Y}} G(Y, Y^{(n)}).
\end{equation*}
Indeed, this procedure yields a non-increasing objective function value due to
\begin{equation*}
    F(Y^{(n+1)}) \leq G(Y^{(n+1)}, Y^{(n)}) \leq G(Y^{(n)}, Y^{(n)}) = F(Y^{(n)}).
\end{equation*}
In the case of NMF with the KL divergence, the goal is to decompose $X\in\R_+^{V\times D}$ into the product of two non-negative matrices $W\in\R_+^{V\times K}$ and $H\in\R_+^{K\times D}$ by minimizing the reconstruction error
\begin{equation*}
    \KL(X\, || \, WH) = \sum_{v,d} x_{vd}\log\frac{x_{vd}}{(WH)_{vd}}- x_{vd} + (WH)_{vd}.
\end{equation*}
Note that it is equivalent to minimize
\begin{equation*}
    F(W, H) = - \sum_{v,d} x_{vd}\log(WH)_{vd} + (WH)_{vd}.
\end{equation*}
For simplicity, we will therefore often refer to $\KL$ and $F$ interchangeably. An auxiliary function of $\KL$ can now be constructed using the following two well-known results.

\begin{lemma}[{\cite[Lemma 3]{lee2000algorithms}}]\label{lemma:aux_function_summand}
    Consider the function $F(W, H) = - \log(WH)_{vd}$ for two indices $v,d$. Then the function $G$ defined by
    \begin{align*}
        \phi'_{vkd} &\coloneqq \frac{w'_{vk} h'_{kd}}{(W'H')_{vd}}, \\
        G((W, H), (W', H')) &= -\sum_{k} \phi'_{vkd} \log \frac{w_{vk}h_{kd}}{\phi'_{vkd}}
    \end{align*}
    is an auxiliary function of $F.$
\end{lemma}
\begin{proof}
    The equality $G((W',H'), (W',H')) = F(W', H')$ is easy to verify and Jensen's inequality implies
    \begin{equation*}
        \log(WH)_{vd} = \log \sum_k \phi'_{vkd} \frac{w_{vk}h_{kd}}{\phi'_{vkd}} \geq \sum_k \phi'_{vkd} \log \frac{w_{vk}h_{kd}}{\phi'_{vkd}}.
    \end{equation*}
\end{proof}

\begin{lemma}\label{lemma:aux_function_combine}
    Consider finitely many function $F_i \colon \mathcal{Y} \rightarrow \R$ with auxiliary functions $G_i,$ and let $c_i\in\R_{+}$. Then $G=\sum_i c_i G_i$ is an auxiliary function of $F=\sum_i c_i F_i.$
\end{lemma}
\begin{proof}
    Follows from the definition of auxiliary functions.
\end{proof}

\begin{corollary}[{\cite[Lemma 3]{lee2000algorithms}}]\label{corollary:aux_function_nmf}
    Consider the function
    \begin{equation*}
        F(W, H) = - \sum_{v,d} x_{vd}\log(WH)_{vd} + (WH)_{vd}.
    \end{equation*}
    Then the function $G$ defined by
    \begin{align*}
        \phi'_{vkd} &= \frac{w'_{vk} h'_{kd}}{(W'H')_{vd}}, \\
        G((W, H), (W', H')) &= - \sum_{v,k,d} x_{vd} \phi'_{vkd} \log \frac{w_{vk}h_{kd}}{\phi'_{vkd}} + \sum_{v,d} (WH)_{vd}
    \end{align*}
    is an auxiliary function of $F$.
\end{corollary}
\begin{proof}
    Apply \Cref{lemma:aux_function_summand} and \Cref{lemma:aux_function_combine}.
\end{proof}

\section{NMF with normalization constraints}\label{appendix:nmf_with_norm_constraints}

In this section, we first proof \Cref{appendix:lemma_connection_nmf_nmfc2}, which states that global optima of NMF with a normalization constraint on $W$ \eqref{optimization_problem_nmf_constraint1} and NMF with a normalization constraint on both $W$ and $H$ \eqref{optimization_problem_nmf_constraint2} are identical up to a scaling of their topic weights by the total word counts $\sum_v x_{vd}$. Second, we provide details to the derivation of the joint update rules of the two NMF problems with normalization constraints. Lastly, we show that the connection from \Cref{appendix:lemma_connection_nmf_nmfc2} extends to iterates and fixed points of their MU algorithms.

\begin{customlemma}{4.2}\label{appendix:lemma_connection_nmf_nmfc2}
    Let $(W, H)$ be a solution of NMF with a normalization constraint on $W$ \eqref{optimization_problem_nmf_constraint1} and let $\lambda_d = \sum_v x_{vd}$. Then $(W, \widetilde{H})$ with $\widetilde{h}_{kd} = h_{kd} / \lambda_d$ is a solution of NMF with a normalization constraint on both $W$ and $H$ \eqref{optimization_problem_nmf_constraint2}.
    Conversely, let $(W, H)$ be a solution of \eqref{optimization_problem_nmf_constraint2}. Then $(W, \widetilde{H})$ with $\widetilde{h}_{kd} = \lambda_d h_{kd}$ is a solution of \eqref{optimization_problem_nmf_constraint1}.
\end{customlemma}
\begin{proof}
    The main ingredient of the proof is that solutions of NMF with the KL divergence \eqref{optimization_problem_nmf} preserve the column sums of $X$. That is, for any solution $(W, H)$ of \eqref{optimization_problem_nmf}, we have
    \begin{equation}\label{sums_preserved}
        \sum_v (WH)_{vd} = \sum_v x_{vd}
    \end{equation}
    \cite[Theorem 6.9]{gillis2020nonnegative}. Now, let $(W, H)$ be a solution of NMF with an additional normalization constraint on $W$. Due to \Cref{lemma_connection_nmf_nmfc1}, it is also a solution of the standard NMF problem \eqref{optimization_problem_nmf}. From \eqref{sums_preserved} and $w_k \in\Delta_{V-1}$, it follows that $\sum_k h_{kd} = \sum_v x_{vd} = \lambda_d.$ Hence, the candidate solution $(W, \widetilde{H})$ defined via $\widetilde{h}_{kd} = h_{kd} / \lambda_d$ indeed fulfills the additional normalization constraint, and we also have
    \begin{equation*}
        \KL(X \, || \, WH) = \sum_{v,d} x_{vd}\log\frac{x_{vd}}{(WH)_{vd}}.
    \end{equation*}
    The result can now be shown using a standard proof by contraction.
\end{proof}

In the following, we provide details to the constrained joint optimization of the joint auxiliary function $G$ of $\KL(X \, || \, WH)$. Recall from \eqref{nmf_joint_auxiliary_function} that $G$ is given by
\begin{align*}
    \phi_{vkd}^{(n)} &= \frac{w_{vk}^{(n)} h_{kd}^{(n)}}{(W^{(n)}H^{(n)})_{vd}}, \\
    G((W, H), (W^{(n)}, H^{(n)})) &= -\sum_{v,k,d} x_{vd} \phi_{vkd}^{(n)} \log \frac{w_{vk}h_{kd}}{\phi_{vkd}^{(n)}} + w_{vk}h_{kd},
\end{align*}
so the derivatives of $G$ are
\begin{align*}
    \partial_{w_{vk}} G((W, H), (W^{(n)}, H^{(n)})) &= -\frac{1}{w_{vk}} \sum_d x_{vd} \phi_{vkd}^{(n)} + \sum_d h_{kd}, \\
    \partial_{h_{kd}} G((W, H), (W^{(n)}, H^{(n)})) &= -\frac{1}{h_{kd}} \sum_v x_{vd} \phi_{vkd}^{(n)} + \sum_v w_{vk}.
\end{align*}
With the additional normalization constraints $w_k \in \Delta_{V-1}$, the method of Lagrange multipliers implies that there exist $\lambda_k\in\R$ with
\begin{equation*}
    \frac{1}{w_{vk}} \sum_d x_{vd} \phi_{vkd}^{(n)} - \sum_d h_{kd} = \lambda_k,
\end{equation*}
which yields
\begin{equation}
    w_{vk} \propto \sum_d x_{vd} \phi_{vkd}^{(n)}.
\end{equation}
Setting the derivative of $G$ with respect to $h_{kd}$ equal to zero and again utilizing $w_k \in \Delta_{V-1}$, we immediately obtain
\begin{equation}\label{eq:update_rule_topic_weights}
    h_{kd} = \sum_v x_{vd} \phi_{vkd}^{(n)}.
\end{equation}
If we additionally require the columns of $H$ to be normalized, i.e., $h_d \in\Delta_{K-1}$, the method of Lagrange multipliers implies that there exist $\lambda_d\in\R$ with
\begin{equation*}
    \frac{1}{h_{kd}} \sum_v x_{vd} \phi_{vkd}^{(n)} - 1 = \lambda_d,
\end{equation*}
which yields
\begin{equation}\label{eq:update_rule_normalized_topic_weights}
    h_{kd} \propto \sum_v x_{vd} \phi_{vkd}^{(n)}.
\end{equation}

Notice that the normalization in \eqref{eq:update_rule_normalized_topic_weights} just scales \eqref{eq:update_rule_topic_weights} by the total word counts $\sum_v x_{vd}$. More precisely, denoting the $n$-th iterate of the NMF with a normalization constraint on $W$ \Cref{alg:mu_nmf_constraint1} by $(W^{(n),\text{\normalfont NMF}}, H^{(n),\text{\normalfont NMF}})$ and the $n$-th iterate of NMF with a normalization constraint on both $W$ and $H$ / PLSA \Cref{alg:mu_nmf_constraint2} by $(W^{(n),\text{\normalfont PLSA}}, H^{(n),\text{\normalfont PLSA}})$, we obtain the following result.

\begin{lemma}\label{lemma:nmf_plsa_iterates}
    Let $X\in\R^{V\times D}_+$ and let $\lambda_d = \sum_v x_{vd}$. Let $W^{(0)}\in\Delta_{V-1}^K, \, H^{(0)}\in\Delta_{K-1}^D$ be normalized initial values. Then for all $n\geq1$
    \begin{align*}
        W^{(n), \text{\normalfont NMF}} &= W^{(n), \text{\normalfont PLSA}}, \\
        h_{kd}^{(n), \text{\normalfont NMF}} &= \lambda_d h_{kd}^{(n), \text{\normalfont PLSA}}, \\
        \sum_k h_{kd}^{(n), \text{\normalfont NMF}} &= \lambda_d.
    \end{align*}
\end{lemma}
\begin{proof}
    Using $\sum_{k,v} x_{vd} \phi_{vkd}^{(n)} = \sum_v x_{vd} = \lambda_d$, the statement follows by induction.
\end{proof}
\begin{corollary}\label{corollary:nmf_plsa_fixed_points}
    Let $X\in\R^{V\times D}_+$ and let $\lambda_d = \sum_v x_{vd}$. Let $(W^{*,\text{\normalfont NMF}}, H^{*,\text{\normalfont NMF}})$ be a fixed point of NMF with a normalization constraint on $W$ \Cref{alg:mu_nmf_constraint1}. Then $(W, H)$ with $W = W^{*,\text{\normalfont NMF}}$ and $h_{kd} = h_{kd}^{*, \text{\normalfont NMF}} / \lambda_d$ is a fixed point of NMF with a normalization constraint on both $W$ and $H$ / PLSA \Cref{alg:mu_nmf_constraint2}. Conversely, let $(W^{*,\text{\normalfont PLSA}}, H^{*,\text{\normalfont PLSA}})$ be a fixed point of \Cref{alg:mu_nmf_constraint2}. Then $(W, H)$ with $W = W^{*,\text{\normalfont PLSA}}$ and $h_{kd} = \lambda_d h_{kd}^{*, \text{\normalfont PLSA}}$ is a fixed point of \Cref{alg:mu_nmf_constraint1}.
\end{corollary}
\begin{proof}
    Follows directly from \Cref{lemma:nmf_plsa_iterates} and $\sum_{k} h_{kd}^{*,\text{\normalfont NMF}} = \sum_v x_{vd} = \lambda_d$ for any fixed point of \Cref{alg:mu_nmf_constraint1}.
\end{proof}

\begin{algorithm}[H]
    \caption{Joint multiplicative updates for NMF \eqref{optimization_problem_nmf_constraint2} / PLSA}\label{alg:mu_nmf_constraint2}
    \begin{algorithmic}
        \Require $X\in\R_{+}^{V\times D}$, initializations $W^{(0)}\in\Delta_{V-1}^K, \, H^{(0)}\in\Delta_{K-1}^D$
        \State $n \gets 0$
        \While {$\KL(X \, || \,W^{(n)}H^{(n)})$ not converged}
        \State $w^{(n+1)}_{vk} \propto w^{(n)}_{vk} \sum_d x_{vd} \frac{ h^{(n)}_{kd} }{ (W^{(n)}H^{(n)})_{vd} }$
        \State $h^{(n+1)}_{kd} \propto h^{(n)}_{kd} \sum_v x_{vd} \frac{ w^{(n)}_{vk} }{ (W^{(n)}H^{(n)})_{vd} }$
        \State $n \gets n + 1$
        \EndWhile
    \end{algorithmic}
\end{algorithm}
\section{Dirichlet--Poisson model}\label{appendix:dirichlet_poisson}
In this section, we derive the variational lower bound of the Dirichlet--Poisson model with a normalized topic matrix $W\in\Delta_{V-1}^K$ described in \Cref{section_connections_lda}. Recall the generative process of the word counts of a single document $d$:
\begin{enumerate}
    \item Sample the topic proportions $h_d \sim \Dirichlet(\alpha)$.
    \item Sample the topic contributions $z_{vkd} \sim \Poisson(w_{vk}h_{kd})$.
    \item $X_{vd} = \sum_k z_{vkd}$.
\end{enumerate}
Under this model, the joint likelihood of the observed counts $x_d$ and the latent variables $(h_d, z_d)$ is given by
\begin{equation*}
    p(x_d, h_d, z_d | W, \alpha) = \Big(\prod_v \mathbbm{1}_{\sum_k z_{vkd} = x_{vd}}\Big) p(h_d | \alpha) \prod_{v,k} p(z_{vkd} | W, h_d),
\end{equation*}
where
\begin{equation*}
    p(z_{vkd} | W, h_d) = (w_{vk}h_{kd})^{z_{vkd}} \frac{e^{-w_{vk}h_{kd}}}{z_{vkd}!}
\end{equation*}
is the probability mass function of the Poisson distribution. Identically to LDA, the introduction of the Dirichlet prior on the topic proportions results in an intractable posterior distribution of the latent variables. We therefore resort to variational inference and approximate the true posterior with a variational distribution $q$ given by
\begin{align*}
    q(h, z | \beta, \phi) = \prod_d q_d(h_d, z_d | \beta_d, \phi_d), \qquad q_d(h_d, z_d | \beta_d, \phi_d) = q(h_d | \beta_d) \prod_v q(z_{vd} | x_{vd}, \phi_{vd}),
\end{align*}
where $\beta_d$ and $\phi_{vd}$ are Dirichlet and multinomial parameters, respectively. We will now compute the individual terms of the per-document variational lower bound
\begin{align*}\label{lda_dirichlet_poisson}
    \log p(x_d | W, \alpha) &\geq \Eq{\log p(x_d, h_d, z_d | W, \alpha)} - \Eq{\log q(h_d, z_d | \beta_d, \phi_d)} \\
    &= \Eq{\log p(h_d | \alpha)} + \Eq{\log p(x_d, z_d | W, h_d)} - \Eq{\log q(h_d)} - \Eq{\log q(z_d)} \\
    &\eqqcolon \mathcal{L}^{\text{DP}}_d (x_d | W, \alpha, \beta_d, \phi_d)
\end{align*}
of the Dirichlet--Poisson model. Plugging in the density of the Dirichlet distribution, the first and third term are given by
\begin{align*}
    \Eq{\log p(h_d|\alpha)} &= \log\Gamma(\textstyle\sum_k \alpha_k) - \displaystyle\sum_k \log\Gamma(\alpha_k) + \sum_k (\alpha_k -1) \Eq{\log h_{kd}}, \\
    \Eq{\log q(h_d)} &= \log\Gamma(\textstyle\sum_k \beta_{kd}) - \displaystyle\sum_k \log\Gamma(\beta_{kd}) + \sum_k(\beta_{kd} -1) \Eq{\log h_{kd}}
\end{align*}
for
\begin{equation}\label{eq:eqlogh_lda}
    \Eq{\log h_{kd}} = \psi(\beta_{kd}) - \psi(\textstyle\sum_{k'} \beta_{k'd}) \eqqcolon \log(\widetilde{h}_{kd})
\end{equation}
and the digamma function $\psi$, cf.\ \cite[A.1, eq.\ (14), (15)]{blei2003latent}. Using the normalization constraints on the topic matrix and topic weights, the second term reduces to
\begin{align*}
    \Eq{\log p(x_d, z_d | W, h_d)} &= \sum_{v,k} \Eq{\log p(z_{vkd} | W, h_d)} \\
    &= \sum_{v,k} \Eq{z_{vkd}} \log(w_{vk}\widetilde{h}_{kd}) - \sum_{v,k} w_{vk}\Eq{h_{kd}} - \sum_{v,k} \Eq{\log z_{vkd}!} \\
    &= \sum_{v,k} \phi_{vkd} x_{vd} \log(w_{vk}\widetilde{h}_{kd}) - 1 - \sum_{v,k} \Eq{\log z_{vkd}!},
\end{align*}
and the last term is given by
\begin{align*}
    \Eq{\log q(z_d)} &= \sum_v \Eq{\log z_{vd}} \\
    &= \sum_v \log x_{vd}! - \sum_{v,k} \Eq{\log z_{vkd}!} + \sum_{v,k} \Eq{z_{vkd}} \log\phi_{vkd} \\
    &= \sum_v \log x_{vd}! - \sum_{v,k} \Eq{\log z_{vkd}!} + \sum_{v,k} \phi_{vkd}x_{vd} \log\phi_{vkd}.
\end{align*}
When combining all terms except data-dependent constants, the summands involving $\Eq{\log z_{vkd}!}$ cancel each other and the per-document variational lower bound simplifies to
\begin{align*}
    \mathcal{L}^{\text{DP}}_d (x_d | W, \alpha, \beta_d, \phi_d) &= \sum_{v,k} x_{vd} \phi_{vkd} \log \frac{w_{vk}\widetilde{h}_{kd}}{\phi_{vkd}} \, + \, \log\Gamma(\textstyle{\sum}_k \alpha_k) - \log\Gamma(\textstyle{\sum}_k \beta_{kd}) \\
    &\phantom{=} + \displaystyle\sum_{k} \Big( \log \Gamma(\beta_{kd}) - \log\Gamma(\alpha_k) + (\alpha_k - \beta_{kd}) \Eq{\log h_{kd}} \Big).
\end{align*}
Finally, the variational lower bound of the Dirichlet--Poisson model is
\begin{equation*}
    \mathcal{L}^{\text{DP}} (X | W, \alpha, \beta, \phi) = \sum_d \mathcal{L}^{\text{DP}}_d (x_d | W, \alpha, \beta_d, \phi_d),
\end{equation*}
which is identical to the variational lower bound \eqref{lda_elbo_concrete} of LDA.
\section{Equivalent generative models}\label{appendix:equivalent_gen_models}
The connections between NMF and PLSA or LDA exist because of the close relationships of their underlying generative models, which we summarize in this section.

We say that a parameter of a model is trivial if the marginal distribution of the observed variables $X$ has a unique global optimum with respect to this parameter that can be immediately estimated from the observations $X$. A model parameter that is not trivial is called non-trivial.

We call two models $m_1$ and $m_2$ equivalent if they have the same non-trivial model parameters $\theta$ and the marginal distributions of the observed variables $X$ are identical up to multiplicative and additive constants after estimating the trivial parameters, i.e.,
\begin{equation*}
    p^{m_1}_\theta(X) = c_1p^{m_2}_\theta(X) + c_2
\end{equation*}
for some data-dependent constants $c_1(X), c_2(X) \in\R.$ It is important to note that potentially existing latent variables are allowed to differ.

\begin{lemma}
    Let the model parameters be $W\in\R^{V\times K}_{+}$ and $h\in\R^{K}_{+}$. Then the following models are equivalent:
    \begin{enumerate}[label=(\roman*)]
      \item $X_v \sim \Poisson((Wh)_v)$
      \item $z_{vk} \sim \Poisson(w_{vk}h_k)$ independent \\[4pt]
      $X_v = \sum_k z_{vk}$
    \end{enumerate}
\end{lemma}
\begin{proof}
    For both models, it holds
    \begin{equation}\label{eq_marginal_poisson}
        p(X_1=x_1, \ldots, X_V=x_V | W,h) = e^{-\sum_v (Wh)_v} \prod_v \frac{1}{x_v!} (Wh)_v^{x_v}.
    \end{equation}
\end{proof}

\begin{lemma}
    Let the model parameters be $W\in\R^{V\times K}_{+}$, $h\in\R^{K}_{+}$ and the total number of observations $N$. Then the following models are equivalent:
    \begin{enumerate}[label=(\roman*)]
        \item $(X_v)_v \sim \Multinomial(N, p_{v}\propto (Wh)_v)$
        \item $(z_{vk})_{vk} \sim \Multinomial(N, p_{vk}\propto w_{vk}h_k)$ \\[4pt]
        $X_v = \sum_k z_{vk}$
    \end{enumerate}
    Additionally, the total number of observations $N$ is a trivial model parameter. If the columns of $W$ are normalized, i.e., $W\in\Delta_{V-1}^K$, then (i) and (ii) are also equivalent to
    \begin{enumerate}[label=(\roman*)]
        \setcounter{enumi}{2}
        \item $\forall n=1,\ldots,N$:
            \begin{itemize}
                \item $z_n \sim \Cat(h)$
                \item $\nu_n \sim \Cat(w_{z_n})$
            \end{itemize}
        $X_v = \sum_n \mathbbm{1}_{\nu_n = v}$
    \end{enumerate}
\end{lemma}
\begin{proof}
    For all three models, it holds
    \begin{equation}\label{eq_marginal_multinomial}
        p(X_1=x_1, \ldots, X_V=x_V | W,h,N) = \mathbbm{1}_{\sum_v x_v = N} \frac{N!}{\prod_v x_v!} \prod_v \Big(\frac{(Wh)_v}{\sum_{v'} (Wh)_{v'}}\Big)^{x_v}.
    \end{equation}
    Note that to see \eqref{eq_marginal_multinomial} for (ii), the trick is to apply the multinomial theorem
    \begin{equation*}
        \sum_{\substack{z\in\N_0^K \\ \sum_k z_k = x}} \frac{x!}{\prod_k z_k!} \prod_k y_k^{z_k} = \Big(\sum_k y_k \Big)^x
    \end{equation*}
    for all $x\in\N_0$ and $y\in\R^K$.
\end{proof}

\begin{lemma}
    Let the model parameters be normalized $W\in\Delta_{V-1}^K$ and $h\in\Delta_{K-1}$, or normalized $W, h$ and the total number of observations $N$. Then the following models are equivalent:
    \begin{enumerate}[label=(\roman*)]
      \item $X_v \sim \Poisson((Wh)_v)$
      \item $z_{vk} \sim \Poisson(w_{vk}h_k)$ independent \\[4pt]
      $X_v = \sum_k z_{vk}$
      \item $(X_v)_v \sim \Multinomial(N, p_{v} = (Wh)_v)$
      \item $(z_{vk})_{vk} \sim \Multinomial(N, p_{vk} = w_{vk}h_k)$ \\[4pt]
        $X_v = \sum_k z_{vk}$
      \item $\forall n=1,\ldots,N$:
        \begin{itemize}
            \item $z_n \sim \Cat(h)$
            \item $\nu_n \sim \Cat(w_{z_n})$
        \end{itemize}
        $X_v = \sum_n \mathbbm{1}_{\nu_n = v}$
    \end{enumerate}
\end{lemma}
\begin{proof}
    With observed $X$, trivially estimated $N$ and normalized model parameters, \eqref{eq_marginal_poisson} and \eqref{eq_marginal_multinomial} are identical up to a constant due to $\sum_v (Wh)_v = 1.$
\end{proof}
\section{Gamma--Poisson model}\label{appendix:gamma_poisson}

In this section, we extend the connections between NMF with a normalization constraint on $W$ \eqref{optimization_problem_nmf_constraint1} and NMF with a normalization constraint on both $W$ and $H$ \eqref{optimization_problem_nmf_constraint2} / PLSA presented in \Cref{appendix:nmf_with_norm_constraints} to their Bayesian counterparts: the Gamma--Poisson and the Dirichlet--Poisson / LDA models. We first introduce the Gamma--Poisson model and derive its variational inference algorithm following \cite{canny2004gap,buntine2005discrete}. Then, we extend the connection between their generative models from \cite[Lemma 1]{buntine2005discrete} to an algorithmic equivalence of their VI algorithms.

In the Gamma--Poisson model we consider, the topics $W\in\Delta_{V-1}^K$ are normalized model parameters. The topic weights $h_d$ are not normalized and Gamma priors $h_{kd} \sim \Gamma(\alpha_k, a_k)$ for additional model parameters $\alpha, a\in\R^K_{> 0}$ are introduced. The generative process of each document $d$ then becomes
\begin{enumerate}
    \item Sample the topic weights $h_{kd} \sim \Gamma(\alpha_k, a_k)$.
    \item Sample the topic contributions $z_{vkd} \sim \Poisson(w_{vk}h_{kd})$.
    \item $X_{vd} = \sum_k z_{vkd}$.
\end{enumerate}

Here and in the following, $\Gamma(\alpha_k, a_k)$ denotes the Gamma distribution with shape parameter $\alpha_k$ and rate parameter $a_k$. The joint likelihood of the observed counts $x_{vd}$ and the latent variables $(h_d, z_d)$ is then given by
\begin{equation*}
    p(x_d, h_d, z_d | W, \alpha) = \Big(\prod_v \mathbbm{1}_{\sum_k z_{vkd} = x_{vd}}\Big) \prod_k p(h_{kd} | \alpha_k, a_k) \prod_{v,k} p(z_{vkd} | W, h_d),
\end{equation*}
where
\begin{align*}
    p(h_{kd} | \alpha_k, a_k) &= \frac{a_k^{\alpha_k}}{\Gamma(\alpha_k)} h_{kd}^{\alpha_k - 1} e^{-a_k h_{kd}}, \\
    p(z_{vkd} | W, h_d) &= (w_{vk}h_{kd})^{z_{vkd}} \frac{e^{-w_{vk}h_{kd}}}{z_{vkd}!}
\end{align*}
denote the density of the Gamma distribution and probability mass function of the Poisson distribution, respectively. Similar to the Dirichlet--Poisson / LDA model, the introduction of the Gamma prior on the topic weights results in an intractable posterior distribution of the latent variables. We apply variational inference and approximate the true posterior with a variational distribution $q$ given by
\begin{align*}
    q(h, z | \beta, b, \phi) &= \prod_d q_d(h_d, z_d | \beta_d, b_d, \phi_d), \\
    q_d(h_d, z_d | \beta_d, b_d, \phi_d) &= \prod_k q(h_{kd} | \beta_{kd}, b_{kd}) \prod_v q(z_{vd} | x_{vd}, \phi_{vd}),
\end{align*}
where $(\beta_{kd}, b_{kd})$ and $\phi_{vd}$ are Gamma and multinomial parameters, respectively. For completeness, we will now compute the individual terms of the per-document variational lower bound
\begin{align*}
    \log p(x_d | W, \alpha, a) &\geq \Eq{\log p(x_d, h_d, z_d | W, \alpha, a)} - \Eq{\log q(h_d, z_d | \beta_d, b_d, \phi_d)} \\
    &= \Eq{\log p(h_d | \alpha, a)} + \Eq{\log p(x_d, z_d | W, h_d)} - \Eq{\log q(h_d)} - \Eq{\log q(z_d)} \\
    &\eqqcolon \mathcal{L}^{\text{GaP}}_d (x_d | W, \alpha, a, \beta_d, b_d, \phi_d)
\end{align*}
of the Gamma--Poisson model. Plugging in the density of the Gamma distribution, the first and third term are given by
\begin{align*}
    \Eq{\log p(h_d|\alpha, a)} &= \sum_k \Big( \alpha_k \log a_k - \log\Gamma(\alpha_k) + (\alpha_k -1) \Eq{\log h_{kd}} - a_k \Eq{h_{kd}} \Big), \\
    \Eq{\log q(h_d)} &= \sum_k \Big( \beta_{kd}\log b_{kd} - \log\Gamma(\beta_{kd}) + (\beta_{kd}-1) \Eq{\log h_{kd}} - b_{kd} \Eq{h_{kd}} \Big)
\end{align*}
for
\begin{align}\label{eq:eqlogh_gap}
    \Eq{\log h_{kd}} &= \psi(\beta_{kd}) - \log b_{kd} \eqqcolon \log(\widetilde{h}_{kd}), \\
    \Eq{h_{kd}} &= \frac{\beta_{kd}}{b_{kd}} \nonumber
\end{align}
and the digamma function $\psi$, cf.\ \cite[eq.\ (14)]{buntine2005discrete}. Identically to the Dirichlet--Poisson model, the second term reduces to
\begin{align*}
    \Eq{\log p(x_d, z_d | W, h_d)} = \sum_{v,k} \phi_{vkd} x_{vd} \log(w_{vk}\widetilde{h}_{kd}) - \sum_k \Eq{h_{kd}} - \sum_{v,k} \Eq{\log z_{vkd}!},
\end{align*}
and the last term is given by
\begin{equation*}
    \Eq{\log q(z_d)} = \sum_v \log x_{vd}! - \sum_{v,k} \Eq{\log z_{vkd}!} + \sum_{v,k} \phi_{vkd}x_{vd} \log\phi_{vkd}.
\end{equation*}
Notice that in contrast to the Dirichlet--Poisson model, the sum $\sum_k \Eq{h_{kd}}$ is not a constant because the topic weights are not normalized. When combining all terms except data-dependent constants, the per-document variational lower bound simplifies to
\begin{align*}
    \mathcal{L}^{\text{GaP}}_d &(x_d | W, \alpha, a, \beta_d, b_d, \phi_d) \\
    &= \sum_{v,k} x_{vd} \phi_{vkd} \log \frac{w_{vk}\widetilde{h}_{kd}}{\phi_{vkd}} - \sum_k \Eq{h_{kd}} + \sum_k \Big(\alpha_k \log a_k - \beta_{kd} \log b_{kd} \Big)\\
    &\phantom{=} + \sum_k \Big(\log\Gamma(\beta_{kd}) - \log\Gamma(\alpha_{k}) + (\alpha_k - \beta_{kd}) \Eq{\log h_{kd}} + (b_{kd} - a_k) \Eq{h_{kd}} \Big).
\end{align*}
Finally, the variational lower bound of the Gamma--Poisson model is
\begin{equation*}
    \mathcal{L}^{\text{GaP}} (X | W, \alpha, a, \beta, b, \phi) = \sum_d \mathcal{L}^{\text{GaP}}_d (x_d | W, \alpha, a, \beta_d, b_d, \phi_d).
\end{equation*}

We again assume that the prior parameters $(\alpha, a)$ are fixed during model training. Optimizing the variational lower bound with respect to the topics $W\in\Delta_{V-1}^K$ and the variational parameters $\phi_{vd}\in\Delta_{K-1}$ works identically to the derivation of the VI algorithm of the Dirichlet--Poisson / LDA model, and their update equations are given by $w_{vk} \propto \sum_d x_{vd} \phi_{vkd}$ and $\phi_{vkd} \propto w_{vk}\widetilde{h}_{kd}$. Moreover, setting the derivatives of the variational lower bound with respect to $\beta_{kd}$ and $b_{kd}$ to zero is equivalent to
\begin{equation*}
    \begin{pmatrix}
        \psi'(\beta_{kd}) & -\frac{1}{b_{kd}}\\
        -\frac{1}{b_{kd}} & \frac{\beta_{kd}}{b_{kd}^2}
    \end{pmatrix}
    \begin{pmatrix}
        \alpha_k + \sum_v x_{vd} \phi_{vkd} - \beta_{kd} \\
        a_k + 1 - b_{kd}
    \end{pmatrix}
    =
    \begin{pmatrix}
        0 \\
        0
    \end{pmatrix},
\end{equation*}
which yields the update equations $\beta_{kd} = \alpha_k + \sum_v x_{vd} \phi_{vkd}$ and $b_{kd} = 1 + a_k$. Rewriting the updates involving $\phi$ in the familiar multiplicative form, we obtain \Cref{alg:gapnmf}. Notice that for fixed prior parameters $(\alpha, a)$ and aside from the trivially estimated variational parameters $b$, the only difference between \Cref{alg:gapnmf} and the Dirichlet--Poisson / LDA \Cref{alg:dpnmf} is the explicit formula for $\Eq{\log h_{kd}}.$

\begin{algorithm}
    \caption{cf.\ \cite[eq.\ (14), (15)]{buntine2005discrete} VI algorithm Gamma--Poisson NMF}\label{alg:gapnmf}
    \begin{algorithmic}
        \Require $X\in\R_{+}^{V\times D}, \, \alpha, a\in\R_{>0}^K$, initializations $W^{(0)}\in\Delta_{V-1}^K, \, \beta^{(0)}\in\R_{>0}^{K \times D}$
        \State $n \gets 0$
        \State $b_{kd} = 1 + a_k$
        \While {$\mathcal{L}^{\text{GaP}}(X | W, \alpha, a, \beta, b, \phi)$ not converged}
        \State $\widetilde{h}_{kd}^{(n)} = \exp(\Eq{\log h_{kd}}) = \frac{\exp(\psi(\beta_{kd}^{(n)}))}{1 + a_k}$
        \State $w^{(n+1)}_{vk} \propto w^{(n)}_{vk} \sum_d x_{vd} \frac{\widetilde{h}_{kd}^{(n)}}{(W^{(n)}\widetilde{H}^{(n)})_{vd}}$
        \State $\beta^{(n+1)}_{kd} = \alpha_k + \widetilde{h}_{kd}^{(n)} \sum_v x_{vd} \frac{w_{vk}^{(n)}}{(W^{(n)}\widetilde{H}^{(n)})_{vd}}$
        \State $n \gets n + 1$
        \EndWhile
    \end{algorithmic}
\end{algorithm}

Before concretizing the algorithmic equivalence of the Gamma--Poisson and Dirichlet--Poisson / LDA VI algorithms, let us recall that the Gamma--Poisson model with rate parameters $a_k=a$ extends the generative model of LDA by also modeling the total counts of words in a document.

\begin{lemma}[{\cite[Lemma 1]{buntine2005discrete}}]
    Let $\alpha\in\R^K_{>0}$ and $a_k=a>0$ be fixed prior parameters. Let the model parameter be $W\in\Delta_{V-1}^K$. Then the following two models are equivalent:
    \begin{enumerate}[label=(\roman*)]
      \item $h_k \sim \Gamma(\alpha_k, a)$ independent, \\[4pt]
      $X_v = \Poisson((Wh)_v)$.
      \item $h \sim \Dirichlet(\alpha)$, \\[4pt]
      $N \sim \PoissonGamma(\sum_k \alpha_k, a)$, \\[4pt]
      $X \sim \Multinomial(N, p=Wh)$.
    \end{enumerate}
\end{lemma}
\begin{proof}
    For simplicity, we denote the value of the density of the Dirichlet distribution at $h\in\Delta_{K-1}$ by $\Dirichlet(\alpha)(h)$, and we use the analogous notation for the probability densities of the Gamma, Poisson, Poisson--Gamma and multinomial distributions, respectively. We also use the abbreviations $\Pois$ and $\mathcal{M}$ to refer to the Poisson and multinomial distributions. With this notation, one can show
    \begin{equation}\label{eq:likelihood_gamma_poisson}
        \begin{split}
            &p(x|W,\alpha,a) \\
            &= \bigg(\int_{\Delta_{K-1}} \Dirichlet(\alpha)(h) \mathcal{M}(\textstyle \sum_v x_v, p=Wh)(x) dh\bigg) \bigg(\displaystyle\int_{\R_{>0}} \Pois(\lambda)(\textstyle\sum_v x_v) \Gamma(\textstyle\sum_k \alpha_k, a)(\lambda) d\lambda\bigg) \\
            &= \bigg(\int_{\Delta_{K-1}} \Dirichlet(\alpha)(h) \mathcal{M}(\textstyle \sum_v x_v, p=Wh)(x) \, dh\bigg) \PoissonGamma(\textstyle\sum_k \alpha_k, a)(\textstyle\sum_v x_v) \\
            &= c(x, \textstyle\sum_k \alpha_k, a) \bigg(\displaystyle\int_{\Delta_{K-1}} \Dirichlet(\alpha)(h) \prod_v \Pois((Wh)_v)(x_v) \, dh\bigg)
        \end{split}
    \end{equation}
    for both generative models. The details of the proof are based on the fact that normalized independent Gamma-distributed random variables with identical rate parameter follow a Dirichlet distribution, see \cite[Lemma 1]{buntine2005discrete}.
\end{proof}
The Bayesian analogue of \Cref{lemma_connection_nmf_nmfc2} is the following result.

\begin{corollary}
    Let $\alpha\in\R^K_{>0}$ and $a_k=a>0$ be fixed prior parameters. Then the Gamma--Poisson and Dirichlet--Poisson / LDA likelihoods have the same optima $W\in\Delta_{V-1}^K.$
\end{corollary}
\begin{proof}
    By \eqref{eq:likelihood_gamma_poisson}, the likelihoods are identical up to a constant $c(X, \textstyle\sum_k \alpha_k, a)$.
\end{proof}

Denoting the $n$-th iterate of the Gamma--Poisson \Cref{alg:gapnmf} by $(W^{(n),\text{\normalfont GaP}}, \beta^{(n),\text{\normalfont GaP}})$ and the $n$-th iterate of Dirichlet--Poisson / LDA \Cref{alg:dpnmf} by $(W^{(n),\text{\normalfont LDA}}, \beta^{(n),\text{\normalfont LDA}})$, the Bayesian analogue of \Cref{lemma:nmf_plsa_iterates} is the following result.

\begin{lemma}\label{lemma:gap_lda_iterates}
    Let $X\in\R^{V\times D}_+$ and let $\alpha\in\R^K_{>0},\, a_k=a>0$ be fixed prior parameters. Let $\lambda_d = \sum_{k'} \alpha_{k'} + \sum_v x_{vd}$. Let $W^{(0)}\in\Delta_{V-1}^K$ and $\beta^{(0)}\in\R_+^{K\times D}$ be arbitrary initial values. Then for all $n\geq1$
    \begin{align*}
        W^{(n), \text{\normalfont GaP}} &= W^{(n), \text{\normalfont LDA}}, \\
        \beta^{(n), \text{\normalfont GaP}} &= \beta^{(n), \text{\normalfont LDA}} = \beta^{(n)}, \\
        \sum_k \beta^{(n)}_{kd} &= \lambda_d.
    \end{align*}
\end{lemma}
\begin{proof}
    The key observation is that the expressions \eqref{eq:eqlogh_lda} and \eqref{eq:eqlogh_gap} of $\Eq{\log h_{kd}}$ in the LDA and Gamma--Poisson \Cref{alg:gapnmf,alg:dpnmf} are identical up to a constant independent of $k$. More precisely, we have
    \begin{align*}
        \widetilde{h}_{kd}^{(0),\text{GaP}} &= \frac{\exp(\psi(\beta_{kd}^{(0)}))}{1+a}, \\
        \widetilde{h}_{kd}^{(0),\text{LDA}} &= \frac{\exp(\psi(\beta_{kd}^{(0)}))}{\exp(\psi(\sum_{k'} \beta_{k'd}^{(0)} ))}.
    \end{align*}
    Consequently, the first iterate of the variational parameters $\phi_{vkd}$ is identical for the GaP and LDA algorithm, which implies $(W^{(1), \text{\normalfont GaP}}, \beta^{(1), \text{\normalfont GaP}}) = (W^{(1), \text{\normalfont LDA}}, \beta^{(1), \text{\normalfont LDA}})$. The induction step and $\sum_k \beta^{(n)}_{kd} = \lambda_d$ follow immediately.
\end{proof}

Finally, the Bayesian analogue of \Cref{corollary:nmf_plsa_fixed_points} is the following result.
\begin{corollary}\label{corollary:gap_lda_fixed_points}
    Let $X\in\R^{V\times D}_+$ and let $\alpha\in\R^K_{>0},\, a_k=a>0$ be fixed prior parameters. Let $(W^*, \beta^*, b^*, \phi^*)$ be a fixed point the Gamma--Poisson VI algorithm. Then $b^*_{kd} = 1 + a$ for all $k,d$ and $(W^*, \beta^*, \phi^*)$ is a fixed point of the Dirichlet--Poisson / LDA VI algorithm. Conversely, let $(W^*, \beta^*, \phi^*)$ be a fixed point of the Dirichlet--Poisson / LDA VI algorithm. Then $(W^*, \beta^*, b^*, \phi^*)$ with $b^*_{kd} = 1 + a$ for all $k,d$ is a fixed point of the Gamma--Poisson VI algorithm.
\end{corollary}
\begin{proof}
    Follows from \Cref{lemma:gap_lda_iterates} and the update equations of $b$ and $\phi$.
\end{proof}
\section{Sparse NMF}\label{appendix:sparse_nmf}

In this section, we first provide details to the derivation of the joint update rules for Lasso-penalized NMF with the KL divergence and $\ell_1$ normalization constraints on columns of $W$. We then prove that Lasso-penalized NMF with the KL divergence and an $\ell_1$ normalization constraint on $W$ \eqref{optimization_problem_nmf_sparse_l1} does not induce any sparsity. Finally, we generalize \cite[Lemma 1 and 2]{marmin2023majorization} and show that any normalization constraints on the columns of W and penalty terms on H are competing regularizers on the two matrices.

Let us denote the objective function of $\ell_1$-penalized NMF with the KL divergence and hyperparameter $\lambda>0$ by
\begin{equation*}
    \KLlambda(X \, || \,  WH) \coloneqq \KL(X \, || \, WH) + \lambda \|H\|_1.
\end{equation*}

Using \Cref{corollary:aux_function_nmf} and $W\in\Delta_{V-1}^K$, a joint auxiliary function of $\KLlambda$ is given by
\begin{align*}
    \phi_{vkd}^{(n)} &= \frac{w_{vk}^{(n)} h_{kd}^{(n)}}{(W^{(n)}H^{(n)})_{vd}}, \\
    G((W, H), (W^{(n)}, H^{(n)})) &= -\sum_{v,k,d} x_{vd} \phi_{vkd}^{(n)} \log \frac{w_{vk}h_{kd}}{\phi_{vkd}^{(n)}} + (1+\lambda) \sum_{k,d} h_{kd}.
\end{align*}
Applying the method of Lagrange multipliers, the auxiliary function can be jointly optimized in $W$ and $H$, and the objective function value is non-increasing under the following joint updates:
\begin{align*}
    w^{(n+1)}_{vk} \propto w^{(n)}_{vk} \sum_d x_{vd} \frac{ h^{(n)}_{kd} }{ (W^{(n)}H^{(n)})_{vd} }, \qquad h^{(n+1)}_{kd} = \frac{1}{1+\lambda} h^{(n)}_{kd} \sum_v x_{vd} \frac{ w^{(n)}_{vk} }{ (W^{(n)}H^{(n)})_{vd} }.
\end{align*}

\begin{customlemma}{6.1}\label{appendix:lemma_lasso_insufficient}
    Let $(W, H)$ be a solution of NMF with an $\ell_1$ normalization constraint on $W$ \eqref{optimization_problem_nmf_constraint1} and let $\lambda > 0$. Then $(W, \tfrac{1}{1+\lambda}H)$ is a solution of Lasso-penalized NMF with an $\ell_1$ normalization constraint \eqref{optimization_problem_nmf_sparse_l1}. Conversely, let $(W, H)$ be a solution of \eqref{optimization_problem_nmf_sparse_l1}. Then $(W, (1+\lambda)H)$ is a solution of \eqref{optimization_problem_nmf_constraint1}.
\end{customlemma}
\begin{proof}
    Notice that for any pair $(W, H)$ with $W\in\Delta_{V-1}^K$, we have
    \begin{equation}\label{appendix:sparse_klnmf_aux_euqations}
        \begin{split}
            \KLlambda(X \, || \, W\tfrac{1}{1+\lambda}H) &= \KL(X \, || \, WH) + \log(1+\lambda)\sum_{v,d} x_{vd}, \\
            \KL(X \, || \, W(1+\lambda)H) &= \KLlambda(X \, || \, WH) - \log(1+\lambda)\sum_{v,d} x_{vd}.
        \end{split}   
    \end{equation}
    The result can now be shown using a proof by contradiction. Let $(W, H)$ be a solution of \eqref{optimization_problem_nmf_constraint1} and assume that $(W, \tfrac{1}{1+\lambda}H)$ is not a solution of $\ell_1$-penalized NMF \eqref{optimization_problem_nmf_sparse_l1}. Then there exists $(W^*, H^*)$ with $\KLlambda(X \, || \, W^*H^*) < \KLlambda(X \, ||  \, W\tfrac{1}{1+\lambda}H)$. Applying \eqref{appendix:sparse_klnmf_aux_euqations}, it follows
    \begin{align*}
        \KL(X \, || \, W^*(1+\lambda)H^*) &= \KLlambda(X \, ||   \, W^*H^*) - \log(1+\lambda)\sum_{v,d} x_{vd} \\
        &< \KLlambda(X \, || \, W\tfrac{1}{1+\lambda}H) - \log(1+\lambda)\sum_{v,d} x_{vd} \\
        &= \KL(X \, || \, WH),
    \end{align*}
    which is a contradiction to $(W, H)$ being a solution of \eqref{optimization_problem_nmf_constraint1}. The proof of the second statement is identical.
\end{proof}

More generally, let us now consider \textit{any} reconstruction error $D(X || WH)$ with $D: \R^{V\times D}_+ \times \R^{V\times D}_+ \to \R_+$, \textit{any} regularizer $R(H)$ with $R: \R^{K \times D}_+ \to \R_+$, and \textit{any} normalization constraint $W\in\mathbb{A}_p$, where $0<p<\infty$ and
\begin{equation*}
    \mathbb{A}_p = \{ W \in\R^{V\times K}_+ \, | \, \forall k: \|w_k\|_p = 1\}
\end{equation*}
denotes the set non-negative matrices with $\ell_p$-normalized columns. The goal is to solve the following sparse NMF problem
\begin{equation}\label{eq:optimization_problem_sparse_nmf_general}
    \min_{W \in\mathbb{A}_p, \, H \in\R^{K\times D}_+} D(X \, || \, WH) + R(H).
\end{equation}
In the special cases of beta divergence reconstruction errors, Lasso- or log-regularizers $R(H)=\lambda \| H \|_1$ or $R(H)=\lambda \sum_{k,d} \log(h_{kd})$, and $\ell_1$-normalized $W\in\mathbb{A}_1$, it has been shown that \eqref{eq:optimization_problem_sparse_nmf_general} is equivalent to an optimization problem without any normalization constraints on $W$ \cite[Lemma 1 and 2]{marmin2023majorization}. We now generalize this result to \eqref{eq:optimization_problem_sparse_nmf_general}.

Let us introduce the normalization matrix $D_p(W) \coloneqq \text{diag}(\|w_1\|_p, \ldots, \|w_K\|_p)\in\R^{K\times K}_+$ and the optimization problem
\begin{equation}\label{eq:optimization_problem_sparse_nmf_general_reformulated}
    \min_{W \in\R^{V\times K}_+, \, H \in\R^{K\times D}_+} D(X \, || \, WH) + R(D_p(W)H).
\end{equation}
Then the optimization problems \eqref{eq:optimization_problem_sparse_nmf_general} and \eqref{eq:optimization_problem_sparse_nmf_general_reformulated} are equivalent in the sense that every solution of \eqref{eq:optimization_problem_sparse_nmf_general} is a solution of \eqref{eq:optimization_problem_sparse_nmf_general_reformulated} and normalized solutions of \eqref{eq:optimization_problem_sparse_nmf_general_reformulated} are solutions of \eqref{eq:optimization_problem_sparse_nmf_general}.

\begin{theorem}[{cf.\ \cite[Lemma 1 and 2]{marmin2023majorization}}]
    Let $X\in\R^{V\times D}_+$. Let $(W, H) \in \mathbb{A}_p \times \R^{K\times D}_+ $ be a solution of \eqref{eq:optimization_problem_sparse_nmf_general}. Then $(W, H)$ is a solution of \eqref{eq:optimization_problem_sparse_nmf_general_reformulated}. Conversely, let $(W, H)\in\R^{V\times K}_+ \times\R^{K\times D}_+$ be a solution of \eqref{eq:optimization_problem_sparse_nmf_general_reformulated}. Then $(W D_p(W)^{-1}, D_p(W)H)$ is a solution of \eqref{eq:optimization_problem_sparse_nmf_general}.
\end{theorem}
\begin{proof}
    Let $(W, H) \in \mathbb{A}_p \times \R^{K\times D}_+ $ be a solution of \eqref{eq:optimization_problem_sparse_nmf_general}. Assume there exists $(W^*, H^*)\in\R^{V\times K}_+ \times\R^{K\times D}_+$ with
    \begin{equation*}
        D(X \, || \, W^*H^*) + R(D_p(W^*)H^*) < D(X \, || \, WH) + R(D_p(W)H) = D(X \, || \, WH) + R(H).
    \end{equation*}
    Then the normalized matrices $(\widetilde{W}, \widetilde{H})$ with $\widetilde{W}=W^* D_p^{-1}(W^*)$ and $\widetilde{H} = D_p(W^*)H^*$ fulfill
    \begin{equation*}
        D(X \, || \, \widetilde{W}\widetilde{H}) + R(\widetilde{H}) < D(X \, || WH) + R(H),
    \end{equation*}
    contradicting that $(W, H)$ is a solution of \eqref{eq:optimization_problem_sparse_nmf_general}. The proof of the second statement is analogous.
\end{proof}

\begin{example}
    Let $0<p,q<\infty$ and consider \eqref{eq:optimization_problem_sparse_nmf_general} with $W\in\mathbb{A}_p$ and $R(H) = \lambda \| H \|_q^q$ for some weight $\lambda > 0$. Then
    \begin{equation*}
        \min_{W \in\mathbb{A}_p, \, H \in\R^{K\times D}_+} D(X \, || \, WH) + \lambda \| H \|_q^q
    \end{equation*}
    is equivalent to
    \begin{equation*}
        \min_{W \in\R^{V\times K}_+, \, H \in\R^{K\times D}_+} D(X \, || \, WH) + \lambda \sum_k \| w_k \|_p^q \|h_k^T\|_q^q.
    \end{equation*}
    This result demonstrates that the normalization constraint on $W$ and the regularizer $R(H)$ are competing penalties.
\end{example}

\begin{example}[\cite{marmin2023majorization}]
    Let $p=1$ and consider \eqref{eq:optimization_problem_sparse_nmf_general} with the beta divergence $D=D_\beta$ for $\beta\in\R$, $W\in\mathbb{A}_1$, and $R(H) = \lambda \| H \|_1$ for some weight $\lambda > 0$.
    Then
    \begin{equation}\label{eq:optimization_problem_sparse_l1_lasso}
        \min_{W \in\mathbb{A}_1, \, H \in\R^{K\times D}_+} D_\beta(X \, || \, WH) + \lambda \| H \|_1
    \end{equation}
    is equivalent to
    \begin{equation*}
        \min_{W \in\R^{V\times K}_+, \, H \in\R^{K\times D}_+} D_\beta(X \, || \, WH) + \lambda \sum_k \| w_k \|_1 \|h_k^T\|_1,
    \end{equation*}
    suggesting that a Lasso penalty on $H$ together with an $\ell_1$ normalization constraint on $W$ induces mutual sparsity \cite{marmin2023majorization}. However, \Cref{lemma_lasso_insufficient} proves that there is no effect on the sparsity in the case $\beta=1$, demonstrating that the effect on the sparsity also depends on the reconstruction error.
\end{example}

\end{document}